\documentclass{article}

\usepackage{amsmath,amsthm}
\usepackage{amsfonts}
\usepackage{amssymb}
\usepackage{mathrsfs}
\usepackage{stmaryrd}
\usepackage{xspace}

\usepackage{enumerate}
\usepackage{hyperref}
\usepackage{color}

\newcommand{\CLS}{\ensuremath{\textsf{CL-S}}^\ast}
\newcommand{\DLS}{\ensuremath{\textsf{DL-S}}^\ast}
\newcommand{\CLP}{\ensuremath{\textsf{CL-P}}^\ast}
\newcommand{\CLF}{\ensuremath{\textsf{CL-F}}^\ast}

\newcommand{\bis}{\leftrightarroweq}

\newcommand{\putawayall}[1]{}

\newcommand{  \dellogic}{\ensuremath{\textsf{DEL} }\xspace}

\newcommand{  \logic}{\ensuremath{\textsf{DL-SM} }\xspace}

\newcommand{  \logicminus}{\ensuremath{\textsf{DL-S}^\ast }\xspace}
\newcommand{  \logicminusone}{\ensuremath{\textsf{DL-S}}\xspace}
\newcommand{  \dlmdlm}{\ensuremath{\textsf{DL-SM}}\xspace}
\newcommand{  \logicminuszero}{\ensuremath{\textsf{DL-S}^0}\xspace}

\newcommand{\compseq}[2]{#1 ; #2}
\newcommand{\choice}[2]{#1 \cup #2}
\newcommand{\iter}[1]{ #1^*}
\newcommand{\test}[1]{ ? #1}

\newcommand{\lup}{ {\Uparrow}}
\newcommand{\ldown}{ {\Downarrow}}
\newcommand{\lright}{ {\Rightarrow}}
\newcommand{\lleft}{ {\Leftarrow} }

\newcommand{\skipact}[1]{   \mathit{nil}_{ #1} }

\newcommand{\gup}[1]{  {\Uparrow}_{ #1} }
\newcommand{\gdown}[1]{  {\Downarrow}_{ #1} }
\newcommand{\gleft}[1]{  {\Leftarrow}_{ #1} }
\newcommand{\gright}[1]{  {\Rightarrow}_{ #1}}

\newcommand{\relAct}[1]{  R_{ #1}}

\newcommand{\here}[1]{  \mathsf{h}_{ #1}}

\newcommand{\Agt}     { \mathit{Agt} }

\newcommand{\putaway}[1]{}
\newcommand{\ijcaiputaway}[1]{}

\newcommand{\see}[2] {\mathsf{ S}_{#1}^{#2}   }

\newcommand{\suchthat}{:}
\newcommand{\functor}[1]{\mathcal{#1}}

\renewcommand{\phi}{\varphi}

\newcommand{\et}{\wedge}

\newcommand{\imp}{\rightarrow} 
\newcommand{\eqv}{\leftrightarrow} 

\newcommand{\ACT}{\mathit{Act}}

\renewcommand{\phi}{\varphi}

\newcommand{\valprop}{\functor{V}}
\newcommand{\posfunct}{\functor{P}}
\newcommand{\natset}{\mathbb{Z}}

\newcommand{\dsucc}{\mathit{succ}}
\newcommand{\dprec}{\mathit{prec}}

\newcommand{\PROP}{\mathit{Atm}}

\newcommand{\AGT}{\mathit{Agt}}

\newcommand{\bnf}{::=}

\newcommand{\ie}{\emph{i.e}.\@\xspace}

\newcommand{  \pdllogic}{\ensuremath{\textsf{PDL} }\xspace}

\newcommand{  \ltllogic}{\ensuremath{\textsf{LTL} }\xspace}

\newcommand{\lang}{ \mathcal{L}_{\logic} }
\newcommand{\langminus}{ \mathcal{L}_{\logicminus} }

\newcommand{  \cllogic}{\ensuremath{\textsf{CL} }\xspace}
\newcommand{  \atllogic}{\ensuremath{\textsf{ATL} }\xspace}
\newcommand{  \stitlogic}{\ensuremath{\textsf{STIT} }\xspace}

\newtheorem{theorem}{Theorem}
\newtheorem{proposition}{Proposition}
\newtheorem{lemma}{Lemma}
\newtheorem{definition}{Definition}%
\newtheorem{corollary}{Corollary}

\newenvironment{proofsketch}{\medskip\noindent \textsc{Sketch of Proof.}}
{\hspace*{\fill}\nolinebreak[2]\hspace*{\fill}$\blacksquare$\medskip}

\newbox\itembox
\def\itemlistlabel#1{#1\hfill}
\def\itemlist#1{\setbox\itembox=\hbox{#1}%
                \list{}{\labelwidth\wd\itembox
                             \leftmargin\labelwidth
                             \advance\leftmargin by\itemindent
                             \advance\leftmargin by\labelsep
                             \let\makelabel\itemlistlabel}}

\begin{document}

\title{
Exploring the bidimensional space:\\
a dynamic logic point of view 
 }

\author{Philippe Balbiani \and David Fern\'andez-Duque \and  Emiliano Lorini
}

\date{
}

\maketitle
\begin{abstract}
We present a family of logics for reasoning about agents' positions and motion in the  plane which have several potential applications in the area of multi-agent systems (MAS), such as multi-agent planning and robotics. The most general logic includes (i) atomic formulas for representing the truth of a given fact or the presence of a given agent at a certain position of the plane, (ii) atomic programs corresponding to the four basic orientations in the plane (up, down, left, right) as well as the four program constructs of propositional dynamic logic $\pdllogic$ (sequential composition, nondeterministic composition, iteration and test). As this logic is not computably enumerable, we study some interesting decidable and axiomatizable fragments of it. We also present a decidable extension of the iteration-free fragment of the logic by special programs representing motion of agents in the plane.
\end{abstract}

\section{Introduction}\label{sec:intro}

Most of existing logics for multi-agent systems (MAS)
including multi-agent epistemic logic \cite{Fagin1995}, multi-agent variants of propositional dynamic logic 
\cite{DBLP:journals/sLogica/SchmidtTH04}
and logics
of action
and strategic reasoning such as $\atllogic$
\cite{DBLP:journals/jacm/AlurHK02}, Coalition Logic  $\cllogic$ \cite{DBLP:journals/logcom/Pauly02}
and $\stitlogic$ \cite{belnap01facing}
are ``ungrounded'' in the sense that their formal semantics are based on abstract primitive notions such as the concept of Kripke model or the concept of possible world (or state).
As a result, there is no direct connection between these \emph{abstract} concepts and
the \emph{concrete} environment in which the agents' interact. 
This kind of grounding problem
 of logics for MAS becomes particularly relevant for robotic applications.
Since robots are situated in spatial environments, in order to make logics for MAS useful for robotics, their semantics have to be grounded on space. Specifically, a formal semantics is required
that provides an explicit representation of the space in which the robots' actions and perceptions are situated.
Some initial steps into the direction
of grounding 
formal semantics
of
logics for MAS
on space
 have done in the recent years. 
 Among them, we should mention logics of multi-agent knowledge in both one-dimensional space and two-dimensional space \cite{GasquetBigBrother,BalbianiFlatland}, 
spatio-temporal logics such as constraint $\ltllogic$ applied to model
2D grid environments \cite{DBLP:conf/atal/AminofMRZ16},
multi-robot task logic based on monadic second-order logic \cite{DBLP:conf/prima/RubinZMA15}
and logics of robot localization \cite{DBLP:journals/sLogica/BelleL16}.
The present paper shares
with these approaches 
 the idea that in order to make existing logics of MAS useful for MAS  applications
 such as multi-agent planning 
 and robotics, 
 their semantics
 should provide an explicit representation
 of the agents' environment.
 
 The main motivation of the present work is
 to provide a logical framework 
 whose language
 and semantics 
are, at the same time, simple
and sufficiently general
to describe 
(i) 
the properties of the spatial environment
in which several agents can move,
and (ii) the consequences of the agents' motion
on such a spatial environment.
To meet this objective, we have decided 
to exploit
the language
of 
propositional dynamic logic $\pdllogic$
as a general formalism 
for representing actions of agents
and their effects,
and to interpret this language
on a simple formal semantics 
of the two-dimensional (2D) space.
The reason why we decided to start from the
2D space is that its representation already presents some 
interesting conceptual
aspects as well as
some
difficulties with respect to the computational properties of the resulting logic. We believe that, before studying action in the 3D space and, more generally, action in 
 $n$-dimensional spaces (with $n > 2$), a
 comprehensive logical theory 
 of action in the 2D space is required.

More concretely, this paper 
presents a family of logics for reasoning about agents' positions and motion in the plane. 
The most general logic,
called Dynamic Logic of  Space $\logicminus$,
is presented in Section  \ref{sec:logic1}.
$\logicminus$
 includes (i) atomic formulas for representing the truth of a given fact
 (\emph{atomic facts})
  or the presence of a given agent at a certain position of the plane (\emph{positional atoms}), (ii) atomic programs corresponding to the four basic orientations in the plane (up, down, left, right) as well as the four program constructs of $\pdllogic$ (sequential composition, nondeterministic composition, iteration and test).
 The logic is proved to be
 non-computably enumerable (non-c.e.)
 and 
its satisfiability problem
 undecidable (Section \ref{section:undecidability}), while its model-checking problem is proved to be decidable in deterministic polynomial time (Section \ref{modelchecking}).
Given the negative properties of $\logicminus$,
we decided to study some interesting decidable
and axiomatizable fragments of it. 
This includes the 
 iteration-free fragment of $\logicminus$ (Section \ref{starfree}) as well
as  a fragment that only allows 
iteration of the same atomic
program (e.g., the action of moving an indefinite number
of times to the right)  and 
has no atomic formulas aside from positional atoms
(Section \ref{canonical}). 
As the logic $\logicminus$
only provides
a static representation
of the 2D space, 
in Section \ref{sec:logic2} we present a decidable extension of its iteration-free fragment
by special programs representing motion of agents in the plane.
Conclusion of the paper 
 (Section \ref{sec:perspectives}) 
 presents
 perspectives of future research
 including integration
 of an epistemic component in the logic
 as well
 as 
of the concept
of coalitional capability 
in the sense
of \cite{DBLP:journals/logcom/Pauly02}.

\section{Space  }\label{sec:logic1}

$\logicminus$ (Dynamic Logic of  Space) is a dynamic logic
which consists
of:
(i) formulas representing the agents'
positions and the truth of facts
in the different positions of the bidimensional space,
and
(ii)
programs allowing to
move from one position
to another position
of the bidimensional space.

\subsection{Syntax }

Assume
a countable set  of atomic propositions $\PROP = \{p,q, \ldots \}$ and
 a finite set of agents $\AGT = \{ 1, \ldots, n \}$.


The language of  $\logicminus$, denoted by $\langminus(\PROP,\AGT)$, is defined by the following grammar in Backus-Naur Form:
\begin{center}\begin{tabular}{lcl}
 $\alpha$  & $\bnf$ & $  \lup   \mid
 \ldown \mid \lright \mid \lleft  \mid
 \compseq{\alpha}{\alpha'}\mid
  \choice{\alpha}{\alpha'}\mid
   \iter{\alpha}\mid
     \test{\varphi} $\\
 $\phi$  & $\bnf$ & $ p   \mid
\here{i} \mid
  \neg\phi \mid \phi \wedge \psi  \mid  [\alpha ]\phi
                        $\
\end{tabular}\end{center}
where $p$ ranges over $\PROP$ and $i $ ranges over $\Agt$. Other Boolean constructions  $\top$, $\bot$, $\vee$, $\imp$ and $\eqv$ are defined from $p$, $\neg$ and $\et$ in the standard way. Instances of $\alpha$
are called
{\em spatial programs.} When there is no risk of confusion we will omit parameters and simply write $\langminus$.
The {\em modal degree} of a formula $\varphi\in \langminus$ (in symbols $\deg(\varphi)$) is defined in the standard way as the nesting depth of modal operators in $\varphi$.
Let $\parallel\varphi\parallel$ denote the size of $\varphi$.
For all (negative or positive) integers $x$, let $\lbrack\lup\rbrack^{x}$ be the modality consisting of $x$ consecutive $\lbrack\ldown\rbrack$ when $x\leq0$, otherwise let $\lbrack\lup\rbrack^{x}$ be the modality consisting of $x$ consecutive $\lbrack\lup\rbrack$.
Similarly for $\lbrack\lright\rbrack^{x}$.

%
The formula $\here{i}$ is read ``the agent $i$  is here'',
whereas
$ [\alpha ]\phi$
has to be read
``$\phi$ is true in the position that is reachable
from the current position through
the program $\alpha$''.

We will also be interested in sublanguages of $\langminus$. Given a set $P$ of atomic propositions, a set $I$ of agents and a set $A$ of spatial programs, we denote the restriction of $\langminus(P,I)$ which only allows programs from $A$ by $\langminus(P,$\linebreak$I,A)$.

\subsection{Semantics  }

The main notion in semantics
is given by the following
concept of spatial model.

\begin{definition}[Spatial model (SM)]

  A spatial model is a tuple $M= (
  \posfunct, \valprop  )$ where:
  \begin{itemize}
  \item $\posfunct : \Agt \longrightarrow \natset \times
  \natset$ and

    \item $\valprop : \natset \times
  \natset
  \longrightarrow 2^\PROP$.

    \end{itemize}

    The set of all spatial models
    is denoted by $\mathbf{M}$.
\end{definition}

For every $(x,y) \in\natset \times
  \natset $,
$\posfunct (i) = (x,y)$
means that the agent $i$
is in the position $(x,y)$,
whereas
$p \in \valprop (x,y) $
means that $p$
is true at the position $(x,y)$. For every $x \in \natset$,
$\dsucc(x)$ denotes the direct successor of $x$ (\ie,
$x  +1$), while 
$\dprec(x)$
denotes the direct predecessor of $x$ (\ie,
$x  -1$).

Formulas
are evaluated with respect
to a spatial model
$M$
and a spatial position $(x,y)$. Below, if $R,S$ are binary relations, $R^\ast$ denotes the transitive, reflexive closure of $R$ and $R\circ S$ the composition of $R$ and $S$.

\begin{definition}[$\relAct{\alpha}$ and truth conditions]\label{accrel1}
Let $M= (
  \posfunct, \valprop  )$
be a spatial
model.
For all spatial programs $\alpha$ and for all formulas $\varphi$,
the binary relation $\relAct{\alpha}$ on $\natset \times
  \natset$ and the truth conditions of
$\varphi$ in $M$ are defined by parallel induction as follows:

\begin{center}\begin{tabular}{lcl}
$\relAct{ \lup } $&$= $&$ \{ ((x,y), (x',y')) \suchthat
x'=x
 \text{ and } y'=\dsucc(y)  \} $\\
$\relAct{ \ldown } $&$= $&$\{ ((x,y), (x',y')) \suchthat
x'=x
 \text{ and } y'=\dprec(y)  \} $\\
$\relAct{ \lright } $&$= $&$ \{ ((x,y), (x',y')) \suchthat
x'=\dsucc(x)
 \text{ and } y'=y  \} $\\
$\relAct{ \lleft } $&$= $&$ \{ ((x,y), (x',y')) \suchthat
x'= \dprec(x)
 \text{ and } y'=y  \} $\\
$\relAct{\compseq{\alpha_1}{\alpha_2}} $&$= $&$ \relAct{\alpha_1} \circ \relAct{\alpha_2} $\\
$\relAct{\choice{\alpha_1}{\alpha_2}} $&$= $&$ \relAct{\alpha_1} \cup \relAct{\alpha_2} $\\
$\relAct{ \iter{\alpha}} $&$= $&$  \iter{(\relAct{ \alpha}) }  $\\
$\relAct{\test{\varphi} } $&$= $&$ \{
((x,y), (x,y))
 \suchthat  M,(x,y) \models \varphi   \} $
\end{tabular}\end{center}

\begin{eqnarray*}
M, (x,y) \models p & \Longleftrightarrow & p \in \valprop( x,y) \\
M, (x,y) \models \here{i}& \Longleftrightarrow &  \posfunct(i) = ( x,y) \\
M, (x,y) \models \neg \varphi & \Longleftrightarrow & M, (x,y)\not \models  \varphi \\
M, (x,y) \models \varphi \wedge \psi & \Longleftrightarrow & M, (x,y) \models \varphi   \text{ and } M, (x,y) \models \psi \\
M, (x,y) \models [\alpha ]\phi & \Longleftrightarrow &
\forall (x',y') \in
\natset \times
  \natset: \text{if }
  (x,y) \relAct{\alpha} (x',y')\\
&&   \text{then }
 M, (x',y') \models \varphi
\end{eqnarray*}
When $ (x,y) \relAct{\alpha} (x',y')$, we will say that position $(x',y')$ is accessible from  position $(x,y)$ by  program $\alpha$.\footnote{
To be more precise, we should define
one relation
$\relAct{\alpha}^{M}$
  per spatial model $M$.
However, we omit the superscript
$M$
since it is clear from the context.}

 \end{definition}

Remark that formulas like $\here{i}$ behave like nominals in hybrid logics~\cite{Blackburn:2000}, i.e. their truth sets are singletons.

We say that
$\varphi\in\langminus$
is valid, denoted by
$\models \varphi$,
if
and only if,  for every spatial
model $M$
and position $(x,y)$,
we have $M,(x,y) \models \varphi$.
We say that
formula
$\varphi\in\langminus$
is satisfiable
if and only if
$\neg \varphi$
is not valid.
\subsection{Bisimulation}
The essential tool we will use to establish our decidability results is the notion of {\em bounded bisimulation.}

\begin{definition}\label{DefBBis}
Fix a set $P$ of atomic propositions, a set $I$ of agents and a set $A$ of spatial programs.
Given spatial models $M_1= (
  \posfunct_1, \valprop_1  )$ and $M_2= (
  \posfunct_2, \valprop_2  )$, $n<\omega$, we define a binary relation $(M_1,\cdot)\leftrightarroweq_n (M_2,\cdot)\subseteq \mathbb Z^2\times \mathbb Z^2$ by induction on $n$ as follows.
  
We set $(M_1,\vec x)\leftrightarroweq_n (M_2,\vec y)$ if
\begin{enumerate}

\item for every $i\in I$, $\vec x= \posfunct_1(i)$ if and only if  $\vec y= \posfunct_2(i)$,

\item for every $p\in P$, $\vec x\in \valprop_1(p)$ if and only if  $\vec y\in \valprop_2(p)$, and
\item if $n>0$, then for every $\alpha\in A$,
\begin{description}
\item[$\mathsf{Forth}_\alpha$] Whenever $x R_\alpha \vec x'$, there is $\vec y'$ such that $\vec y R_\alpha \vec y'$ and $(M_1,\vec x')\leftrightarroweq_{n-1} (M_2,\vec y')$, and
\item[$\mathsf{Back}_\alpha$] Whenever $\vec y R_\alpha \vec y'$, there is $\vec x'$ such that $\vec x R_\alpha \vec y'$ and $(M_1,\vec x')\leftrightarroweq_{n-1} (M_2,\vec y')$.
\end{description}
\end{enumerate}
\end{definition}

We may just write $\vec x \bis_n \vec y$ instead of\linebreak $(M ,\vec x)\leftrightarroweq_{n} (M ,\vec y)$.
The following is then standard:

\begin{lemma}\label{LemmBisim}
Fix a set $P$ of atomic propositions, a set $I$ of agents and a set of programs $A$. If $M_1= (
  \posfunct_1, \valprop_1  )$ and $M_2= (
  \posfunct_2, \valprop_2  )$ are spatial models and $\varphi\in\langminus(P,I,A)$ has modal degree at most $n$, then whenever $(M_1,\vec x)\leftrightarroweq_{n} (M_2,\vec y)$, it follows that $M_1,\vec x\models \varphi$ if and only if $M_2,\vec y\models \varphi$.
\end{lemma}

\section{Undecidablity}\label{section:undecidability}

This section presents 
results about undecidability for
the satisfiability problem of  $\langminus(\PROP,\AGT)$-formulas.
Products of linear logics are logics with two (or more) modalities, interpreted over structures very similar to spatial models. 
Their formulas are equivalent to $\langminus(\PROP,\AGT)$-formulas over the class of all spatial models
and are often undecidable \cite{Gabbay:Kurucz:Wolter:Zakharyaschev:2003,Marx97Compass,Reynolds:Zakharyaschev:2001}.
This suggests that the satisfiability problem of  formulas in $\langminus(\PROP,\AGT)$, as well as some proper fragments, is undecidable as well.
The idea is to allow actions only along the horizontal and vertical axes, which following \cite{MarxVenema1997} we call the `compass directions'.
To be precise, we define the {\em language of compass logic of space} by ${\mathcal L}_{\CLS}(\PROP,\AGT)={\mathcal L}_{\DLS}(\PROP,\AGT,C)$, where
\[
C= \{  \lup   ,\ldown , \lright , \lleft  ,
 \iter\lup,\iter\ldown , \iter\lright , \iter\lleft \}
.\]
As before, we may omit the parameters $\PROP,\AGT$ when this does not lead to confusion. By ${\mathcal L}_{\CLF}$ (the language of compass logic of facts) we denote the special case where $\AGT=\varnothing$, and similarly ${\mathcal L}_{\CLP}$ (the language of compass logic of positions) denotes the case where $\PROP=\varnothing$.

We start with
the following undecidability result
for the satisfiability
problem
of the latter.
\begin{theorem}\label{TheoUndec:a}
The set of valid formulas of ${\mathcal L}_{\CLF}$ is not computably enumerable.
\end{theorem}
\begin{proof}
This follows from Theorem 5.38 in \cite{Gabbay:Kurucz:Wolter:Zakharyaschev:2003}, which states (in their notation) that ${\sf PTL}_{\Box\circ}\times {\sf PTL}_{\Box\circ}$ is not c.e. But this is a notational variant of a fragment of ${\CLF }$, where $\circ_1\approx [\lright]$, $\Box_1\approx [\lright^\ast]$, $\circ_2\approx [\lup]$, and $\Box_2\approx [\lup^\ast]$.
\end{proof}

We remark that we only need two of the four compass directions for this proof, provided they are perpendicular. As a corollary, we obtain undecidability of the larger logic.

\begin{corollary}\label{CorUndec}
The set of valid formulas of ${\mathcal L}_{\DLS}$ is not computably enumerable.
\end{corollary}

In order to study model-checking,
we need a finite representation of spatial models.
To this aim,
we introduce the following definition of
\emph{bounded}
spatial model
of size $n$.
For $(x,y)\in \mathbb Z^2$, write $|(x,y)|\leq n$ iff $|x|\leq n$ and $|y|\leq n$.

\begin{definition}[Bounded spatial model (BSM)]
Let $n$ be a nonnegative integer.
A  spatial model $M= (
  \posfunct, \valprop  )$ is said to be $n$-bounded iff for all $i\in Agt$, $|\posfunct(i)|\leq n$ and for all $(x,y)\in\mathbb Z^2\times \mathbb Z^2$, if $|(x,y)|\not\leq n$ then $\valprop(x,y)=\emptyset$.

\end{definition}

Observe that while the interpretations of variables are bounded, the frame itself is not; we still interpret formulas over $\mathbb Z\times\mathbb Z$.

As for the class of all models, the restriction to bounded models gives rise to an undecidable set of valid formulas of ${\mathcal L}_{\DLS}$-formulas:

\begin{theorem}\label{TheoUndec:b}
The set of formulas of ${\mathcal L}_{\CLF}$ valid over the class of bounded spatial models is not computably enumerable.
\end{theorem}
\begin{proofsketch}
This essentially follows from Corollary 7.18 in \cite{Gabbay:Kurucz:Wolter:Zakharyaschev:2003}, which in their notation states that ${\sf Log}\{\langle \mathbb N,{\geq}\rangle\times \mathbb N,{\geq}\rangle\}$ is not c.e. As above, this is a notational variant of a fragment of ${\CLF }$, where $\Box_1\approx [\lright^\ast]$ and $\Box_2\approx [\lup^\ast]$, although interpreted over frames of the form $\{ 0,\hdots n\}\times \{0,\hdots,m\}$. That it is not c.e. is obtained by reducing the halting problem for Turing machines to ${\sf Log}\{\langle \mathbb N,{\geq}\rangle\times \mathbb N,{\geq}\rangle\}$, representing finite computations as finite models. Minor adjustments of this construction can be used, instead, to represent finite computations as {\em bounded} models.
\end{proofsketch}

As before, the undecidability of the set of formulas of ${\mathcal L}_{\DLS}$ valid over the class of bounded spatial models follows.
There are different ways to get out of the undecidability of the satisfiability problem
of ${\mathcal L}_{\DLS}$-formulas as highlighted by
Corollary
\ref{CorUndec}.
One possibility
is to consider the 
star-free fragment of ${\mathcal L}_{\DLS}$.
Another possibility
is to study 
fragments of 
${\mathcal L}_{\DLS}$
that omit atomic propositions and allow only nominals.
These two possibilities
are explored, respectively,
in Sections
\ref{starfree}
and
\ref{canonical}.

\section{Model-checking}\label{modelchecking}

The model-checking problem for ${\mathcal L}_{\CLS}(\PROP,\AGT)$ is the following:
let $\varphi \in {\mathcal L}_{\CLS}(\PROP,\AGT)$,
let $n$ be a nonnegative integer,
let $M$
be an $n$-bounded spatial model
and let $(x,y) \in  \mathbb Z\times
   \mathbb Z $,
   is it the case
   that $M, (x,y) \models \varphi$?

In this section we will show that the model-checking problem for ${\mathcal L}_{\CLS}(\PROP,\AGT)$ is in {\sc PTime}. We use techniques similar to those used for proving that, e.g., model-checking for ordinary modal logic logic or for $\sf CTL$ is also in {\sc PTime} \cite{model:checking:CTL}, but there are some subtleties in dealing with the state-space being infinite (even if the valuations are bounded).
%
%
%
%
\begin{lemma}\label{square:n:n:bisimulation}
Let $n,d$ be nonnegative integers.
Suppose that $x,x',y,y'\in \mathbb Z$ are such that one of the following conditions holds:
\begin{itemize}
\item $x>n+d+1$, $x^{\prime}=x-1$ and $y^{\prime}=y$,
\item $x<-n-d-1$, $x^{\prime}=x+1$ and $y^{\prime}=y$,
\item $y>n+d+1$, $x^{\prime}=x$ and $y^{\prime}=y-1$,
\item $y<-n-d+1$, $x^{\prime}=x$ and $y^{\prime}=y+1$.
\end{itemize}
Then, for {\em any} $n$-bounded model $M$, we have that
\[(M,(x,y))\leftrightarroweq_{d}(M,(x^{\prime},y^{\prime})).\]
\end{lemma}
\begin{proof}
Left to the reader.
\end{proof}
Hence,
\begin{lemma}\label{square:n:n}
Let $\varphi$ be a ${\mathcal L}_{\CLS}(\PROP,\AGT)$-formula, $n$ be a nonnegative integer and $M$ be an $n$-bounded model.
For all integers $x,y$, we have:
\begin{itemize}
\item if $x>n+\deg(\varphi)+1$ then $M,(x,y)\models\varphi$ iff $M,(x-1,y)\models\varphi$,
\item if $x<-n-\deg(\varphi)-1$ then $M,(x,y)\models\varphi$ iff $M,(x+1,y)\models\varphi$,
\item if $y>n+\deg(\varphi)+1$ then $M,(x,y)\models\varphi$ iff $M,(x,y-1)\models\varphi$,
\item if $y<-n-\deg(\varphi)-1$ then $M,(x,y)\models\varphi$ iff $M,(x,y+1)\models\varphi$.
\end{itemize}
\end{lemma}
\begin{proof}
By Lemmas~\ref{LemmBisim} and~\ref{square:n:n:bisimulation}.
\end{proof}
Now, for all ${\mathcal L}_{\CLS}(\PROP,\AGT)$-formulas $\varphi$ and for all integers $z$, let $z_{\varphi}$ be the integer defined by cases as follows:
\begin{description}
\item[Case $|z|\leq n+\deg(\varphi)+1$:]
In that case, let $z_{\varphi}=z$.
\item[Case $z<-n-\deg(\varphi)-1$:]
In that case, let $z_{\varphi}=-n-\deg(\psi)-1$.
\item[Case $z>n+\deg(\varphi)+1$:]
In that case, let $z_{\varphi}=n+\deg(\psi)+1$.
\end{description}
The reader may easily verify that for all ${\mathcal L}_{\CLS}(\PROP,\AGT)$-formulas $\varphi$ and for all integers $z$, $|z_{\varphi}|\leq n+\deg(\varphi)+1$.
Now, given a ${\mathcal L}_{\CLS}(\PROP,\AGT)$-formula $\varphi$, let $(\varphi_{1},\ldots,\varphi_{N})$ be an enumeration of the set of all $\varphi$'s subformulas.
Let us assume that for all $a,b\in\{1,\ldots,N\}$, if $\varphi_{a}$ is a strict subformula of $\varphi_{b}$ then $a<b$.
For all $a\in\{1,\ldots,N\}$ and for all $(x,y)\in\mathbb Z^2\times \mathbb Z^2$, if $|(x,y)|\leq n+\deg(\varphi_{a})+1$ then we will associate a truth value $tv(a,x,y)$ by case as follows:
\begin{description}
\item[Case $\varphi_{a}=p$:]
In that case, let $tv(a,x,y)=$``$(x,y)\in V(p)$''.
\item[Case $\varphi_{a}=h_{i}$:]
In that case, $tv(a,x,y)=$``$(x,y)=P(i)$''.
\item[Case $\varphi_{a}=\bot$:]
In that case, let $tv(a,x,y)=\bot$.
\item[Case $\varphi_{a}=\neg\psi$:]
Let $b\in\{1,\ldots,a\}$ be such that $\psi=\varphi_{b}$.
Remind that $b<a$.
In that case, if $tv(b,x,y)=\bot$ then let $tv(a,x,y)=\top$ else let $tv(a,x,y)=\top$.
\item[Case $\varphi_{a}=\psi\vee\chi$:]
Let $b,c\in\{1,\ldots,a\}$ be such that $\psi=\varphi_{b}$ and $\chi=\varphi_{c}$.
Remind that $b,c<a$.
In that case, if $tv(b,x_{\psi},y_{\psi})=\bot$ and $tv(c,x_{\chi},y_{\chi})=\bot$ let $tv(a,x,y)=\bot$ else let $tv(a,x,y)=\top$.
\item[Case $\varphi_{a}=\lbrack\lright\rbrack\psi$:]
Let $b\in\{1,\ldots,a\}$ be such that $\psi=\varphi_{b}$.
Remind that $b<a$.
In that case, let $tv(a,x,y)=tv(b,(x+1)_{\psi},y_{\psi})$.
%
%
%
%
%
%
%
%
\item[Cases $\varphi_{a}=\lbrack\lup\rbrack\psi$, $\varphi_{a}=\lbrack\lleft\rbrack\psi$ and $\varphi_{a}=\lbrack\ldown\rbrack\psi$:]
Similar to\linebreak the previous case.
%
\item[Case $\varphi_{a}=\lbrack\iter\lright\rbrack\psi$:]
Let $b\in\{1,\ldots,a\}$ be such that $\psi=\varphi_{b}$.
Remind that $b<a$.
In that case, if $tv(b,z_{\psi},y_{\psi})=\bot$ for some integer $z\geq x$ then let $tv(a,x,y)=\bot$ else let $tv(a,x,y)=\top$.
%
%
%
%
%
%
%
%
\item[Cases $\varphi_{a}=\lbrack\iter\lup\rbrack\psi$, $\varphi_{a}=\lbrack\iter\lleft\rbrack\psi$ and $\varphi_{a}=\lbrack\iter\ldown\rbrack\psi$:]
Similar to the previous case.
\end{description}
Obviously, within a polynomial time with respect to $\parallel\varphi\parallel$, one can deterministically compute the truth values $tv(a,x,y)$ for $a\in\{1,\ldots,N\}$ and for $(x,y)\in\mathbb Z^2\times \mathbb Z^2$ such that $|(x,y)|\leq n+\deg(\varphi_{a})+1$.
Consequently,
\begin{theorem}
The model-checking problem for ${\mathcal L}_{\CLS}(\PROP,$\linebreak$\AGT)$ is decidable in deterministic polynomial time.
\end{theorem}
%
%
%
%

\section{Star-free fragments}\label{starfree}

In this section and the next, we identify two decidable fragments.
The first is obtained by restricting the language to ${\mathcal L}_{\logicminusone}(\PROP,\AGT)$, as given by the following grammar:

\begin{center}\begin{tabular}{lcl}
 $\alpha$  & $\bnf$ & $  \lup   \mid
 \ldown \mid \lright \mid \lleft  \mid
 \compseq{\alpha}{\alpha'}\mid
  \choice{\alpha}{\alpha'}\mid
     \test{\varphi} $\\
 $\phi$  & $\bnf$ & $ p   \mid
\here{i} \mid
  \neg\phi \mid \phi \wedge \psi  \mid  [\alpha ]\phi
 
                        $\
\end{tabular}\end{center}

We denote the set of valid formulas of ${\mathcal L}_{\logicminusone}(\PROP,\AGT)$ by $\logicminusone$. The second is the fragment ${\mathcal L}^0_{\logicminusone}(\PROP,\AGT)$ given by:

\begin{center}\begin{tabular}{lcl}
 $\alpha$  & $\bnf$ & $  \lup   \mid
 \ldown \mid \lright \mid \lleft  $\\
 $\phi$  & $\bnf$ & $ p   \mid
\here{i} \mid
  \neg\phi \mid \phi \wedge \psi  \mid  [\alpha ]\phi
 
                        $\
\end{tabular}\end{center}
The corresponding set of valid formulas will be denoted $\logicminuszero$.
Note that ${\mathcal L}_{\logicminusone}(\PROP,\AGT)$ can be reduced to\linebreak ${\mathcal L}^0_{\logicminusone}(\PROP,\AGT)$:

\begin{lemma}
Every formula $\varphi\in {\mathcal L}_{\logicminusone}(\PROP,\AGT)$ is equivalent to some $\varphi^0\in {\mathcal L}^0_{\logicminusone}(\PROP,\AGT)$.
\end{lemma}

\proof
It suffices to observe that the following are valid:
\begin{align*}
[\compseq{\alpha}{\alpha'}]\psi&\leftrightarrow [\alpha][\alpha']\psi\\
[\choice{\alpha}{\alpha'}]\psi&\leftrightarrow [\alpha]\psi\wedge [\alpha']\psi\\
[\test{\theta}]\psi&\leftrightarrow (\theta\rightarrow\psi).
\end{align*}
With these validities, any formula of ${\mathcal L}_{\logicminusone}(\PROP,\AGT)$ can be recursively reduced to an equivalent formula in the language ${\mathcal L}^0_{\logicminusone}(\PROP,\AGT)$.
\endproof

Our decidability proof will be based on a small model property, obtained by truncating a larger model. Fix a natural number $n$.
Given a model $M=(\posfunct,\valprop)$, we define $M\upharpoonright n=(\posfunct \upharpoonright n,\valprop\upharpoonright n)$.
\begin{itemize}

\item $(\posfunct\upharpoonright n)(i)=
\begin{cases}
\posfunct(i)&\text{if $|\posfunct(i)|\leq n$;}\\
(n+1,0)&\text{otherwise.}
\end{cases}
$

\item $(\valprop \upharpoonright n)(p)=\valprop (p)\cap \big ( [-n,n]\times[-n,n] \big )$.
\end{itemize}
Observe that $M\upharpoonright n$ is $(n+1)$-bounded.
As a result, when one restricts the discussion to the set of all programs of ${\mathcal L}^0_{\logicminusone}(\PROP,\AGT)$,

\begin{lemma}\label{LemmTruncate}
For all $\vec x\in\mathbb Z^2\times \mathbb Z^2$, if $|\vec x|\leq m\leq n$, then $(M,\vec x)\leftrightarroweq_{n-m}(M\upharpoonright n,\vec x)$.
\end{lemma}

\proof
The proof proceeds by a standard induction on $m$. The atoms and position clauses are trivial since $x\leq n$ and the values of atomic propositions is not changed. For the inductive case, consider (for example) $\alpha={\lright}$. Then, if $\vec x=(x_0,x_1)$, $\vec x R_\lright \vec y$ if and only if $\vec y=(x_0+1,x_1)$. Clearly $|\vec y| \leq m+1$, so that by the induction hypothesis, $(M,\vec x)\leftrightarroweq_{n-m-1}(M\upharpoonright n,\vec y)$, as needed.
\endproof

With this we obtain our first decidability result.

\begin{theorem}\label{firstdecid}
The logics $\logicminuszero,\logicminusone$ are decidable. In particular, $\logicminuszero$ is in {\sc NP}.
\end{theorem}

\proof
Since $\logicminusone$ can be reduced to $\logicminuszero$, it suffices to show that the latter is decidable. Suppose that $\varphi$ is satisfied on some model $M$; without loss of generality, we can assume that $\varphi$ is satisfied on the origin. Let $n$ be the modal degree of $\varphi$. By Lemma \ref{LemmTruncate}, $(M,\vec 0)\leftrightarroweq_n (M\upharpoonright n,\vec 0)$, so by Lemma \ref{LemmBisim}, $\varphi$ is also satisfied on $(M\upharpoonright n,\vec 0)$. It follows that $\varphi$ is satisfiable if and only if it is satisfiable on the class of models such that $\posfunct$ and $\valprop$ are both $(n+1)$-bounded, so it remains to enumerate all such models and check whether any of them satisfy $\varphi$. Note that the size of any $(n+1)$-bounded model is $o(n^2)$, so the complexity bound for $\logicminuszero$ follows.
\endproof

Observe that it does not follow from our techniques that $\logicminusone$ is in {\sc NP}, since the reduction procedure is not polynomial.

Now, our aim in this section will be to completely axiomatize $\logicminuszero$.
In this respect, we need the following axioms and inference rules:
\begin{itemize}
\item All axioms and inference rules saying that $\lbrack\lup\rbrack$, $\lbrack\ldown\rbrack$, $\lbrack\lright\rbrack$ and $\lbrack\lleft\rbrack$ are normal modalities,
\item $\lbrack\alpha\rbrack\varphi\leftrightarrow\langle\alpha\rangle\varphi$ for each $\alpha\in\{\lup,\ldown,\lright,\lleft\}$,
\item $\varphi\rightarrow\lbrack\lup\rbrack\langle\ldown\rangle\varphi$ and $\varphi\rightarrow\lbrack\ldown\rbrack\langle\lup\rangle\varphi$,
\item $\varphi\rightarrow\lbrack\lright\rbrack\langle\lleft\rangle\varphi$ and $\varphi\rightarrow\lbrack\lleft\rbrack\langle\lright\rangle\varphi$,
\item $\lbrack\alpha_{1}\rbrack\lbrack\alpha_{2}\rbrack\varphi\leftrightarrow\lbrack\alpha_{2}\rbrack\lbrack\alpha_{1}\rbrack\varphi$ for each $\alpha_{1},\alpha_{2}\in\{\lup,\ldown,\lright,\lleft\}$,
\item $h_{i}\rightarrow\lbrack\lup\rbrack^{x}\lbrack\lright\rbrack^{y}\neg h_{i}$ for each nonnegative integers $x,y$ such that $x\not=0$ or $y\not=0$.
\end{itemize}
We will say that a formula $\varphi\in {\mathcal L}^0_{\logicminusone}(\PROP,\AGT)$ is {\em derivable} if it belongs to the least set of ${\mathcal L}^0_{\logicminusone}(\PROP,\AGT)$-formulas containing the above axioms and closed under the above inference rules.
\begin{theorem}\label{completeness:D:L:S:0}
let $\varphi$ be an ${\mathcal L}^0_{\logicminusone}(\PROP,\AGT)$-formula.
The following conditions are equivalent: (i)~$\varphi$ is derivable; (ii)~$\varphi$ is valid.
\end{theorem}
\begin{proof}
(i)$\Rightarrow$(ii): It suffices to check that all axioms are valid and that all inference rules preserve valitity.
\\
(ii)$\Rightarrow$(i): Suppose $\varphi$ is not derivable.
Let $d$ denote the modal degree of $\varphi$.
By Lindenbaum's Lemma, let $\Gamma$ be a maximal consistent set of formulas such that $\varphi\not\in\Gamma$.
Remark that for all $i\in Agt$, if $h_{i}\not\in\Gamma$ then there exists at most one pair $(x,y)$ of (negative or positive) integers such that $\lbrack\lup\rbrack^{x}\lbrack\lright\rbrack^{y} h_{i}\in\Gamma$.
Let $Agt(\Gamma)$ be the set of all $i\in Agt$ such that $\lbrack\lup\rbrack^{x}\lbrack\lright\rbrack^{y} h_{i}\in\Gamma$ for some pair $(x,y)$ of integers such that $|(x,y)|\leq d$.
Let $M=(\posfunct,\valprop)$ be the spatial model defined as follows:
\begin{itemize}
\item For all $i\in Agt$, if $i\in Agt(\Gamma)$ then let $\posfunct(i)$ be the unique pair $(x,y)$ of integers such that $\lbrack\lup\rbrack^{x}\lbrack\lright\rbrack^{y} h_{i}\in\Gamma$, else let $\posfunct(i)$ be $(d+1,0)$,
\item for all pairs $(x,y)$ of integers, if $|(x,y)|\leq d$ then let $\valprop(x,y)=\{p\in Atm:\ \lbrack\lright\rbrack^{x}\lbrack\lup\rbrack^{y}p\in\Gamma\}$, else let $\valprop(x,y)=\emptyset$.
\end{itemize}
The reader may easily prove by induction on $\psi$ that if $\psi$ is a subformula of $\varphi$ then for all pairs $(x,y)$ such that $|(x,y)|\leq deg(\varphi)-deg(\psi)$, $M,(x,y)\models\psi$ iff $\lbrack\lright\rbrack^{x}\lbrack\lup\rbrack^{y}\psi\in\Gamma$.
Since $\varphi\not\in\Gamma$, therefore $M,(0,0)\not\models\varphi$.
Thus, $\varphi$ is not valid.
\end{proof}
%
%
%
%

\section{Compass logic of positions}\label{canonical}
Next we consider the fragment ${\mathcal L}_{\CLP}$, defined in Section \ref{section:undecidability}.
Since there are no atomic propositions, models are somewhat simpler.

\begin{definition}
A {\em position model} is a function $\mathcal P\colon \Agt\to \mathbb Z\times\mathbb Z$.
\end{definition}

That is, a position model is just a spatial model without a valuation for atomic propositions. As we will show, position models do not need to have big `gaps' if we only care about satisfiability of ${\mathcal L}_{\CLP}$-formulas. This will give us a small model property.

\begin{definition}
Let $\mathcal P$ be a position model. A {\em vertical gap} is a set $G=[a,b]\times\mathbb Z$ such that for all $i\in \Agt$, $\mathcal P(i)\not\in G$. If $(x,y)\in G$, we say that the {\em depth of $(x,y)$ in $G$} is $\min(x-a,b-x)$, and $G_m$ denotes the set of elements of depth at least $m$; observe that $G_0=G$, and $G_m$ is also a gap when non-empty. The {\em removal of $G$} is the function $\rho$ given by $\rho(x,y)=(x',y)$ where $x'=x$ if $x\leq a$, $x'=\min(a,x-(b-a))$ otherwise.

A {\em horizontal gap} is defined analogously, but is of the form $\mathbb Z\times [a,b]$.
The depth and the removal are defined analogously as well.
\end{definition}

In this section we use $\bis_n$ for $n$-bisimilarity with respect to all basic relations of ${\mathcal L}_{\CLP}$.

\begin{lemma}\label{LemmStripBisim}
Let $G=[a,b]\times\mathbb Z$ be a vertical gap and $\mathcal P$ be a position model. Then, if $(x,y),(x',y)\in G_m$, it follows that $\mathcal P,(x,y)\bis_m \mathcal P,(x',y)$.

The analogous claim holds for horizontal gaps.
\end{lemma}

\proof
We proceed by induction on $m$. The atomic clauses are straightforward since, if $(x,y),(x',y)\in G_0=G$, then they satisfy no atoms.

For the other clauses, assume the claim inductively for $m$, and suppose that $(x,y),(x',y)\in G_{m+1}$. Any `vertical' program ($\lup,\ldown,\iter\lup,\iter\ldown$) stays within $G_{m+1}\subseteq G_m$ so we may immediately apply the induction hypothesis. For example, if $(u,v)R_{\iter\lup}(x,y)$, then $u=x$ and $v\geq y$; hence, $(x',y)R_{\iter\lup}(x',v)$ and by the induction hypothesis, $(u,v)=(x,v)\bis_{m} (x',v)$. The `back' clauses and the rest of the vertical programs are entirely symmetrical.

Next consider a `horizontal' program: $\lleft,\lright,\iter\lleft,\iter\lright$. By symmetry, we will only consider the `forth' clauses of the `right' programs. We have that $R_{\lright}$ is a function; specifically, $R_{\lright}(x,y)=(x+1,y)$. Observe that $(x+1,y)\in G_m$, and similarly $(x'+1,y)\in G_m$. But, by the induction hypothesis, $(x+1,y)\bis_{m} (x'+1,y)=R_{\lright}(x',y)$, as needed.

Now suppose that $(x,y) R_{\iter\lright}(u,v)$, so that $u\geq x$ and $v=y$. We consider two cases. If also $u\geq x'$, then we also have that $(x',y) R_{\iter\lright}(u,v)$, and we may use the same witness. Otherwise, $x\leq u<x'$, which means that $(u,v)\in G_{m+1}\subseteq G_m$, so by the induction hypothesis $(x,v)\bis_{m} (x',y)$. But also, $(x',y) R_{\iter\lright}(x',y)$, and we can use it as our witness.

As mentioned, the other clauses are entirely symmetrical and left to the reader. By induction on $m$, the claim follows. The analogous claim for horizontal gaps is also entirely analogous.
\endproof

\begin{lemma}\label{LemmRemoveBisim}
Let $G=[a-m,b+m]\times\mathbb Z$ be a vertical gap and $\rho $ the removal of $G_m$. Then, $\mathcal P,\vec x \bis _m \rho\mathcal P,\rho(\vec x)$.
\end{lemma}

\proof
We prove, by induction on $k\leq m$, that $\mathcal P,\vec x \bis _k \rho\mathcal P,\rho(\vec x)$. For $k=0$ this is clear, since if $(x,y)\in G_m$, no nominal occurs on $(x,y)$ or on $\rho(x,y)=(a,y)$. Otherwise, $(x,y)=\mathcal P(i)$ if and only if $\rho(x,y)=\rho\mathcal P(i)$. 

Now, assume the claim for $k$, and let $\rho(x,y)=(x',y)$. The `forth' clauses for $\alpha\in\{R_\lup,R_\ldown,R^\ast_\lup,R^\ast_\ldown\}$ follow by observing that if $(x,y)R_\alpha(u,v)$, then $\rho(x,y)R_\alpha\rho(u,v)$; for example, if $\alpha={\ldown}$, then we must have $u=x$ and $v=y-1$, and since $\rho$ fixes the $y$ coordinate we have that if $\rho(x,y)=(x',y)$, then $\rho(u,v)=(x',y-1)$, as needed. Similarly, for the `back' clause, if $\rho(x,y)=(x',y)$ and $(x',y)R_\alpha (u,v)$, we must have $u=x'$ and can readily observe that $(x,y) R_\alpha (x,v)$ and $\rho(x,v)=(x',v)$, so that by the induction hypothesis, $\mathcal P,(x,v)\bis _k \rho\mathcal P,\rho(x',v)$.

Next we look at $\alpha\in \{\lleft,\lright,\iter\lleft,\iter\lright\}$. First, we check the `forth' clauses. If $(x,y)R_\lleft(u,v)$, then $u=x-1$ and $v=y$. If $x\not\in (a,b]$, then it readily follows that $\rho(x,y)R_\lleft\rho(u,v)$, and we may use the induction hypothesis. If instead $x\in (a,b]$, then $\rho(u,v)=\rho(x,y)=(a,y)$. However, $R_\lleft(a,y)=(a-1,y)\in G_{m-1}$, so by the induction hypothesis and Lemma \ref{LemmStripBisim},
\[\rho \mathcal P,(a-1,y)\bis _{k}\mathcal P,(a-1,y)\bis _{k} \mathcal P,(x-1,y),\]
as needed. For $\alpha=\iter\lleft$, suppose $(x,y)R_{\iter\lleft}(u,v)$. Then, $y=v$, and since $\rho$ is non-decreasing on the first component, we also have $\rho(x,y)R_{\iter\lleft}\rho(u,v)$. The cases for the `right' programs are similar.

Finally, we check the `back' clauses for the horizontal programs. Observe that $R_\lleft,R_\lright$ are functional, so the `forth' and `back' clauses are equivalent. Hence we consider only $R_{\iter\lleft},R_{\iter\lright}$. If $(x',y)R_{\iter\lleft}(u,y)$, then consider two cases. If $u\leq a$, then $\rho(u,y)=(u,y)$ and $u\leq x'\leq x$, so we have that $(x,y)R_{\iter\lleft}(u,y)$ and we may use the induction hypothesis on $(u,y)$. If $u>a$, then $\rho(u+b-a,y)=(u,y)$, and we may use the induction hypothesis on $(u+b-a,y)$. But note that, in this case, we must have that $x=x'+b-a$, so $(x,y)R_{\iter\lleft}(u+b-a,y)$.

Finally, if $(x',y)R_{\iter\lright}(u,y)$, again consider two cases. If $u< a$, then $\rho(u,y)=(u,y)$ and $u\geq x'=x$, so we have that $(x,y)R_{\iter\lright}(u,y)$ and we may use the induction hypothesis on $(u,y)$. If $u\geq a$, then $\rho(u+b-a,y)=(u,y)$, and we may use the induction hypothesis on $(u+b-a,y)$. Note that, in this case, $x\leq x'+b-a \leq u+b-a$, so $(x,y)R_{\iter\lright}(u,y)$, as needed.

The case for a horizontal gap is similar.
\endproof

\begin{theorem}
If $\varphi\in \mathcal L_{\CLP}$ is satisfiable, it is satisfiable on a position model where all coordinates of positions are bounded by $2(|\varphi|+1)^2$.
\end{theorem}

\proof
Assume that $\varphi$ is satisfied on some position model $\mathcal P$. Suppose that $x_1\leq \hdots\leq x_n$ are the $x$-coordinates of all positions of agents such that $\here i$ appears in $\varphi$, together with the evaluation point, $(0,0)$ (note that $n\leq |\varphi|+1$). If for some $i<n$ we have that $x_{i+1}-x_i>2(|\varphi|+1)$, then $G=(x_{i},x_{i+1})\times\mathbb Z$ is a horizontal strip with $G_{|\varphi|}$ having width at least two, so that its removal is not the identity.

Now, if the $x_i$'s are not bounded by $2(|\varphi|+1)^2$, note that such a gap must exist so we can remove it. After enough iterations, we can bound all $x_i$'s. Then we proceed to bound the vertical components analogously.
\endproof

\begin{theorem}
Satisfiability of ${\mathcal L}_{\CLP}$-formulas is decidable in {\sc NP}.
\end{theorem}

\proof
We can decide the satisfiability of $\varphi$ by guessing a position model $\mathcal P$ with all coordinates bounded by $2(|\varphi|+1)^2$ and model-checking whether $\varphi$ holds at $(0,0)$.
\endproof

\section{Space and motion }\label{sec:logic2}

$\logicminusone$ studied in the previous 
sections
is a logic
for representing static properties
of the bidimensional space.
Specifically, in 
$\logicminusone$, positions of agents
in the space do not change. 
The aim of this section
is to extend ${\mathcal L}_{\logicminusone}(\PROP,\AGT)$ by programs describing the agents' motions in
the bidimensional space.
We assume that agents act
in a synchronous way
(i.e., they act in parallel).
We call the resulting language ${\mathcal L}_{\dlmdlm}(\PROP,\AGT)$ and the resulting logic 
$\logic$ (Dynamic Logic of Space
and Moving).

\subsection{Syntax }

In ${\mathcal L}_{\dlmdlm}(\PROP,\AGT)$,
 agent
$i$ is associated
 with her corresponding
repertoire
of actions
$\ACT_i =  \{ \gup{i }, \gdown{i },
\gleft{i }, \gright{i }, \skipact{i } \}$.
$\gup{i }$
is agent $i$' action
of moving up,
$\gdown{i }$
is agent $i$' action
of moving down,
$\gleft{i }$
is agent $i$' action
of moving left,
$\gright{i }$
is agent $i$' action
of moving right
and $\skipact{i }$
is agent $i$'s action of doing nothing.

 The set of joint of actions is defined
 to be $\Delta =  \prod_{i \in \AGT}  \ACT_i  $.
 Elements of $\Delta$
 are denoted by $\delta, \delta', \ldots$
 For every $\delta \in \Delta$,
 $\delta_i $
 denotes the element in $\delta $
 corresponding to agent $i$.

Since the logic 
$\logicminus$
is undecidable, we start from
its decidable star-free fragment as
the basis of our dynamic extension
by
programs describing
the agents' motion.

The language, denoted by $\lang(\PROP,\AGT)$, is defined by the following grammar in Backus-Naur Form:
\begin{center}\begin{tabular}{lcl}
 $\alpha$  & $\bnf$ & $  \lup   \mid
 \ldown \mid \lright \mid \lleft  \mid
 \compseq{\alpha}{\alpha'}\mid
  \choice{\alpha}{\alpha'}\mid
     \test{\varphi} $\\
 $\beta$  & $\bnf$ & $  \delta \mid
 \compseq{\beta}{\beta'}\mid
  \choice{\beta}{\beta'}\mid
     \test{\varphi}
    $\\
 $\phi$  & $\bnf$ & $ p   \mid
\here{i} \mid
  \neg\phi \mid \phi \wedge \psi  \mid  [\alpha ]\phi
  \mid [\beta]\phi
                        $\
\end{tabular}\end{center}
where $p$ ranges over $\PROP$ and $i $ ranges over $\Agt$. Instances of $\beta$
are called {\em motion programs.}

\subsection{Semantics  }

The semantics
is a model update semantics
as in the style of dynamic epistemic logic ($ \dellogic$)~\cite{DELbook}. 

\begin{definition}[$\relAct{\beta}^{(x,y)}$
and truth conditions]\label{accrel2}
Let $M \in \mathbf{M}$
be a spatial program. 
  For all motion programs $\beta$, for all formula $\varphi$
  and for all positions  $(x,y) $, 
  the binary relation 
  $\relAct{\beta}^{(x,y)}$
  on $\mathbf{M} \times \mathbf{M}$
  and the truth conditions of $\varphi$
  in $M$ are defined by parallel induction as follows.
  (We only give the truth condition for $[\beta] \varphi$ as the truth conditions
  for the boolean constructs and for  $[\alpha] \varphi$
   are as in $\logicminus$):
  \begin{center}\begin{tabular}{lcl}
$ \relAct{\delta}^{(x,y)} 
$&$= $&$ \{ (M, M') \suchthat
\valprop'=\valprop
 \text{ and }
 \forall
 i \in \Agt,
  \posfunct'(i) =\posfunct^{ \delta_i }
 (i)  \} $\\
$\relAct{\compseq{\beta_1}{\beta_2}}^{(x,y)} $&$= $&$ \relAct{\beta_1}^{(x,y)} \circ \relAct{\beta_2}^{(x,y)} $\\
$\relAct{\choice{\beta_1}{\beta_2}}^{(x,y) }$&$= $&$ \relAct{\beta_1}^{(x,y) }\cup \relAct{\beta_2}^{(x,y)} $\\
$\relAct{\test{\varphi} }^{(x,y) }$&$= $&$ \{
(M, M)
 \suchthat  M,(x,y) \models \varphi   \} $
\end{tabular}\end{center}
where:
 \begin{center}\begin{tabular}{lll}
$\posfunct^{ \delta_i }
 (i) = (x,\dsucc(y)) $& if &
 $\delta_i = \gup{i }$ and $\posfunct(i)=(x,y)$ \\
 $\posfunct^{ \delta_i }
 (i) = (x,\dprec(y)) $& if &
 $\delta_i = \gdown{i }$ and $\posfunct(i)=(x,y)$ \\
 $\posfunct^{ \delta_i }
 (i) = (\dsucc(x),y) $& if &
 $\delta_i = \gright{i }$ and $\posfunct(i)=(x,y)$ \\
 $\posfunct^{ \delta_i }
 (i) = (\dprec(x),y) $& if &
 $\delta_i = \gleft{i }$ and $\posfunct(i)=(x,y)$ \\
  $\posfunct^{ \delta_i }
 (i) = (x,y) $& if &
 $\delta_i = \skipact{i }$ and $\posfunct(i)=(x,y)$
\end{tabular}\end{center}

  \begin{eqnarray*}
M, (x,y) \models [\beta ]\phi & \Longleftrightarrow &
\forall (M,M') \in
\mathbf{M} \times
 \mathbf{M}: \text{if }
  M \relAct{\beta} M'\\
&&  
 \text{then }
 M', (x,y) \models \varphi
\end{eqnarray*}

\end{definition}

Definitions
of validity and satisfiability
for $\logic $
generalize
those for
$\logicminus $
in a straighforward manner.

\subsection{Decidability 
and axiomatization }\label{decibspacemotion}

The aim of this section
is to show how the satisfiability problem
of $\logic$ can be reduced to the   satisfiability problem
of $\logicminusone$. 
Given the decidability result
and the complete
axiomatization  for the latter
of Section \ref{starfree},
this reduction will provide a decidability
result as well as an axiomatization for the former.

 \begin{proposition}\label{redaxioms}
 The following $\lang(\PROP,\AGT)$-formulas are valid:
 \begin{align}
   [\compseq{\alpha}{\alpha'} ] \varphi \leftrightarrow 
  [\alpha ] [\alpha' ] \varphi \\
    [\choice{\alpha}{\alpha'} ] \varphi \leftrightarrow 
 ( [\alpha ]  \varphi \wedge  [\alpha' ]  \varphi ) \\
  [\compseq{\beta}{\beta'} ] \varphi \leftrightarrow 
  [\beta ] [\beta' ] \varphi \\
    [\choice{\beta}{\beta'} ] \varphi \leftrightarrow 
 ( [\beta ]  \varphi \wedge  [\beta' ]  \varphi ) \\
      [\test{\varphi} ] \psi \leftrightarrow 
 ( \varphi \rightarrow \psi ) \\
 [\delta ] p \leftrightarrow p \\
  [\delta ] \here{i} \leftrightarrow  [F_i (\delta) ] \here{i} \\
 [\delta ] \neg \phi \leftrightarrow    \neg [\delta ]  \phi \\
[\delta ] ( \phi \wedge \psi)   \leftrightarrow   (   [\delta ]  \phi   \wedge   [\delta ]  \psi )\\
[\delta ] [ \alpha  ]\phi  \leftrightarrow        [\alpha  ] [\delta ] \phi 
\end{align}
with $ \alpha \in \{ 
\lup   , \ldown , \lright , \lleft \}$
and
where the function $F_i$
is defined as follows:
 \begin{align*}
F_i (\delta ) = \lup \text{ if } \delta_i = \gdown{i }  \\
F_i (\delta ) =  \ldown  \text{ if } \delta_i = \gup{i }  \\
F_i (\delta ) = \lright  \text{ if } \delta_i = \gleft{i }  \\
F_i (\delta ) = \lleft  \text{ if } \delta_i = \gright{i }  \\
F_i (\delta ) = \test{\top}     \text{ if } \delta_i = \skipact{i } 
\end{align*}
 \end{proposition}

 As the following rule of replacement of equivalents preserves validity:
  \begin{align}
  \frac{\psi_1 \leftrightarrow \psi_2}{ \varphi \leftrightarrow \varphi[ \psi_1 / \psi_2]}
   \end{align}
 the equivalences
of Proposition \ref{redaxioms} together with this allow to find for every $\lang(\PROP,\AGT)$-formula  an
equivalent  formula
of 
${\mathcal L}_{\logicminusone}(\PROP,\AGT)$ studied in Section \ref{starfree}. Call $\mathit{red}$ the mapping which iteratively applies the
equivalences of Proposition \ref{redaxioms} from the left to the right, starting from one of the innermost modal
operators. $\mathit{red}$ pushes the dynamic operators $[\beta ]$ inside the formula, and finally
eliminates them when facing an atomic formula
The mapping $\mathit{red}$ is inductively defined by:
{\footnotesize
  \begin{align*}
1. &  \mathit{red}(p) = p \\
2. &  \mathit{red}(\here{i}) = \here{i} \\
3. &  \mathit{red}( \neg\phi) = \neg\mathit{red}( \phi)\\
4. &  \mathit{red}(\phi \wedge \psi) = \mathit{red}(\phi) \wedge \mathit{red}( \psi)\\
5. &  \mathit{red}([\alpha]\phi) = 
[\alpha ] \mathit{red}(\phi) \text{ with }
\alpha \in \{ 
\lup   , \ldown , \lright , \lleft \}\\
6. &  \mathit{red}([ \compseq{\alpha}{\alpha'} ]\phi) = 
[\alpha ][\alpha' ] \mathit{red}(\phi) \\
7. &  \mathit{red}([ \choice{\alpha}{\alpha'} ]\phi) = 
([\alpha ]\mathit{red}(\phi) \wedge [\alpha' ] \mathit{red}(\phi)) \\
8. &  \mathit{red}([\test{\varphi} ] \psi) = 
\mathit{red}(\neg (\varphi \wedge \neg \psi) ) \\
9. &  \mathit{red}([\delta]p) = p\\
10. &  \mathit{red}([\delta]\here{i}) = [F_i (\delta) ] \here{i}\\
11. &  \mathit{red}([\delta]\neg \phi) = \mathit{red}(\neg[\delta] \phi)\\
12. &  \mathit{red}([\delta ] ( \phi \wedge \psi)) = \mathit{red}(   [\delta ]  \phi   \wedge   [\delta ]  \psi )\\
13. &  \mathit{red}([\delta ] [ \alpha  ]\phi) = \mathit{red}(   [ \alpha  ] [\delta ] \phi )\ \text{with} \ \alpha \in \{ 
\lup   , \ldown , \lright , \lleft \}\\
14. &  \mathit{red}([ \compseq{\beta}{\beta'} ]\phi) = 
[\beta ][\beta' ] \mathit{red}(\phi) \\
15. &  \mathit{red}([ \choice{\beta}{\beta'} ]\phi) = 
([\beta ]\mathit{red}(\phi) \wedge [\beta' ] \mathit{red}(\phi)) \\
   \end{align*}
 
}

We can state the following proposition.

  \begin{proposition}
  Let $\varphi \in \lang(\PROP,\AGT)$.
  Then, $\varphi \leftrightarrow \mathit{red}(\phi)$
  is valid.
  \end{proposition}

Decidability of the 
satisfiability problem
of $\logic$
follows straightforwardly from the decidability of the star-free fragment
${\mathcal L}_{\logicminusone}(\PROP,\AGT)$
of $\logicminus$ (Theorem \ref{firstdecid}).
 Indeed, $\mathit{red}$ provides an effective procedure for reducing a formula $\varphi$
 in $\lang(\PROP,\AGT)$ into an equivalent 
 formula $\mathit{red}(\phi)$
 in ${\mathcal L}_{\logicminusone}(\PROP,\AGT)$.

 \begin{theorem}
 The logic $\logic$ is decidable.
 \end{theorem}

Thanks to
the completeness result
for the  star-free fragment
 of
 $\logic$ and the reduction axioms
 of Proposition \ref{redaxioms},
 we can state the following theorem.

 \begin{theorem}
 The logic $\logic$ is completely axiomatized
 by the axioms
 and rules of inference
 of the star-free fragment
 of
 $\logic$  given in Section \ref{starfree}, 
 the valid formulas
 of Proposition \ref{redaxioms} and the rule of replacement of equivalents. 
 \end{theorem}

\section{Perspectives }\label{sec:perspectives}

\newcommand{\clop}[1]{  \langle\![  #1 ]\!\rangle }

Before
we discuss two  perspectives
for the extension of the logic
$\logicminus$
and 
$\logic$
by concepts
of perceptual knowledge
and  coalitional capability.

\paragraph{Perceptual knowledge} $\logicminus$
and
$\logic$ support reasoning about
properties
of the 2D space as well as about
positions and motion
of agents in the 2D space. 
However, an agent in the space
does not only move but also
sees where other agents are,
how the space around her is,
what other agents do, etc.
More generally,
agents in the space have perceptual
knowledge (i.e., knowledge based on what they see).
We want to
propose here a simple extension 
of  $\logicminus$
and
$\logic$ 
by modal operators of perceptual 
knowledge.
Specifically, 
we consider epistemic-like
operators of type $\see{ i }{ k } $
  describing 
what an agent
could 
see 
from her current position,
if she had a range of vision
of size
  $k \in \mathbb{N}$.
  An agent's range of vision
  of size  
  $k  $
corresponds to the square centered at the agent's position
with side length equal to $2\times k$. 
We call the latter
agent $i$'s \emph{neighborhood} of size $k$.


%
%
%
%
%
%

%
%
%
%

In order to provide an interpretation of the operator
$\see{ i }{ k }  $,
 the following concept
of indistinguishibility is required.
Let $i \in \Agt$
and
let $M = (
  \posfunct, \valprop  )$ and $ M' = (
  \posfunct', \valprop'  )  $ be two spatial models. We say that
  $M$
  and
 $M'$
  are \emph{indistinguishable}
  for
  agent $i$
  given
  her current position
 and her range of vision
of size
  $k$, denoted
  by $M \sim_i^k M'$,
  if and only if:
\begin{center}\begin{tabular}{lll}
$\valprop'( x,y ) $& $=$ &
 $\valprop( x,y ) $  \\
$\posfunct'(j)  $& $=$ &
 $\posfunct(j) $
\end{tabular}\end{center}
  for all $(x,y)
  \in \natset \times
  \natset $
  and for all   $j \in \Agt$
  such that
  $(x,y) \in \mathcal{D}(i,k)$
  and   $ \posfunct(j) \in \mathcal{D}(i,k)$
  with
  \begin{align*}
  \mathcal{D}(i,k) = & \{ (x',y')  : \posfunct_x(j) - k \leq x' \leq  \posfunct_x(j) + k \text{ and } \\
& \posfunct_y(j) - k \leq y' \leq  \posfunct_y(j) + k  \}
  \end{align*}
where $\posfunct_x(j)$
and $\posfunct_y(j)$
 are, respectively, the $x$-coordinate
 and the $y$-coordinate in $\posfunct(j)$.

This notion of indistinguishibility
is essential
to provide a truth condition of the
 formula
$\see{ i }{ k } \varphi $
that has to be read
``if agent $i$ had a range of vision 
 of size
  $k$,
then $i$ could see that $\varphi$
is true from her current position''.
Let $M$
be a spatial model
and let
$(x,y) \in\natset \times
  \natset $. Then:
\begin{eqnarray*}
M, (x,y) \models \see{ i }{ k  }
\phi & \Longleftrightarrow &
\forall M' \in
\mathbf{M}: \text{if }
  M \sim_i^k M'\\
&&   \text{then }
 M', \posfunct(i) \models \varphi
\end{eqnarray*}

It is easy to check that $\sim_i^k$
is an equivalence relation. Thus, the operator $ \see{ i }{ k  }$
satisfies all S5 principles.
$ \see{ i }{ k  }$ satisfies additional principles that are
proper to its spatial interpretation.
For instance,
let
\begin{align*}
  \mathcal{P}rg(k) =&  \{   \lleft^{h   } \lup^h : 0 \leq h \leq k  \}  \cup
   \{ \lright^{h   } \lup^h: 0 \leq h \leq k  \}  \cup\\
   & \{ \lleft^{h   } \ldown^h: 0 \leq h \leq k  \}  \cup
     \{ \lright^{h   } \ldown^h: 0 \leq h \leq k  \}  
\end{align*}
be  the set of  spatial programs that 
allow to reach \emph{all and only} those points
in an agent's neighborhood of size $k$.
Then, under the previous interpretation
of the operator $\see{ i }{ k  }$, the following formulas become valid
for every $\alpha \in  \mathcal{P}rg(k) $:
\begin{align}
 \here{i}  \rightarrow  ([\alpha ] p \leftrightarrow  \see{ i }{ k  } [\alpha ] p )\\
  \here{i}  \rightarrow  ([\alpha ] \here{j} \leftrightarrow  \see{ i }{ k  } [\alpha ] \here{j} )
\end{align}
This means that if an
agent $i$ has a range of vision of size
  $k$, then
she can perceive
  all facts that are true 
  and all agents that are positioned
    in her neighborhood of size $k$.
    The following formula
    is an example of instance of the previous validity:
    \begin{align}
 \here{i}  \rightarrow  ([\lup]  \here{j} \leftrightarrow  \see{ i }{ 1  } [\lup ] \here{j} )
\end{align}
The latter means that if 
agent $i$ has a range of vision of size
  $1$ then,
  agent $j$ is above her iff 
agent $i$  perceives this.

We postpone to future work
a study of the complexities 
of  model-checking
and of decidability 
for the extensions
by epistemic operators $\see{ i }{ k  }$
of
the different
logics
presented in the paper.

\paragraph{Coalitional capability}

$\logic$ provides an interesting basis
for
the development 
of a logic
of coalitional capabilities in the two-dimensional space.
We take the concept
of `coalitional capability' 
in the sense
of Coalition Logic  $\cllogic$ \cite{DBLP:journals/logcom/Pauly02}.
Specifically, we say that
coalition $C$ has the capability
of ensuring $\varphi$, denoted by $\clop{C}\varphi$,
if and only if
``there exists
a joint action $\delta_C$
of coalition $C$
such that, by performing it,
 outcome $\varphi$ will be ensured,
no matter what the agents
outside $C$
decide to do''. 
The extension
of $\logic$
by coalitional capability
operators 
 $\clop{C}$
 is rather simple,
 as the agents' action repertoires
 only includes
  the four basic movements in the plane ($\gup{i }, \gdown{i },
\gleft{i }$ and $ \gright{i }$)
and the action of doing nothing ($\skipact{i }$).

 Following  Section \ref{sec:logic2},
 for every coalition $C \subseteq \AGT$
 we  define
its
set of joint of actions $\Delta_C =  \prod_{i \in C}  \ACT_i  $
and denote
 elements of $\Delta_C$ by $\delta_C, \delta_C', \ldots$
 Then, 
 the 
 truth condition of the operator
 $\clop{C}$ goes as follows: $M, (x,y) \models \clop{C}\varphi  $ if and only if $
\exists \delta_C \in \Delta_C$ such that
\[\forall \delta_{\AGT \setminus C }' \in \Delta_{\AGT \setminus C } :
 M, (x,y) \models [\delta_C,\delta_{\AGT \setminus C }' ] \varphi 
.\]

Since $\delta_C$
and $\delta_{\AGT \setminus C }'$
are finite, 
$ \clop{C}\varphi$
is expressible in $\logic$  but at the price
of an exponential blowup
in the size of the formula $ \varphi$. 

It is easy to check that 
the operator
$\clop{C}$
satisfies the following
basic principles
of the coalitional capability
operator by
\cite{DBLP:journals/logcom/Pauly02}:
    \begin{align}
\neg \clop{C} \bot \\
 \clop{C} \top \\
 \neg \clop{\emptyset} \neg \varphi 
 \rightarrow \clop{\AGT} \varphi  \\
 \clop{C}(\varphi \wedge \psi)
 \rightarrow \clop{C}\varphi \\
  (\clop{C_1}\varphi \wedge
  \clop{C_2}\psi)\rightarrow
  \clop{C_1 \cup C_2}
  (\varphi \wedge \psi) \notag \\ 
  \text{ if }
  C_1 \cap C_2 = \emptyset\\
  \frac{\varphi \leftrightarrow \psi}{ \clop{C} \varphi \leftrightarrow  \clop{C} \psi}
\end{align} 
$\clop{C}$ satisfies additional principles
that are proper to its spatial interpretation. 
For instance, it is easy to check that, under
the previous interpretation,
the following two formulas
become valid: 
    \begin{align}
\neg \clop{ C  } \here{i} \text{ if } i \not \in C \\
([\lup] \here{i} \vee [\lright]\here{i} \vee [\ldown]\here{i} \vee 
[\lleft]\here{i})
\rightarrow \clop{  \{ i  \} } \here{i} 
\end{align}
The two validities captures
the basic idea
that an agent has exclusive control
of her position in the sense that:
(i) if  coalition $C$ does not include
agent $i$
then $C$ cannot force $i$ to be ``here'', and
(ii) agent $i$ has the capability to move ``here''
if she is ``around''.

We postpone to future work a more systematic analysis
of 
the basic principles
of the operator $\clop{C}$
as well as 
a
study of 
a strategic capability
operator
in the sense of 
$\atllogic$
\cite{DBLP:journals/jacm/AlurHK02}
based on the semantics
of the logic $\logic$.


\end{document}